\documentclass{article}
\usepackage{kr}

\usepackage[hyphens]{url}  
\usepackage{times}
\usepackage{soul}
\usepackage{url}
\usepackage[hidelinks]{hyperref}
\usepackage[utf8]{inputenc}
\usepackage[small]{caption}
\usepackage{graphicx}
\usepackage{amsmath}
\usepackage{amsthm}
\usepackage{booktabs}
\usepackage{algorithm}
\usepackage{epstopdf}
\usepackage{multirow}
\usepackage[all]{xy}
\usepackage[noend]{algpseudocode}
\usepackage{verbatim}
\usepackage{enumitem}

\DeclareMathOperator*{\argmax}{arg\,max}
\DeclareMathOperator*{\argmin}{arg\,min}

\newtheorem{definition}{Definition}
\newtheorem{lemma}{Lemma}
\newtheorem{corollary}{Corollary}
\newtheorem{example}{Example}
\newtheorem{proposition}{Proposition}

\newcommand{\mc}[1]{\ensuremath{\mathcal{#1}}}
\newcommand{\mt}[1]{\ensuremath{\mathtt{#1}}}

\newcommand{\ti}[1]{\ensuremath{\textit{#1}}}

\newcommand{\labelFunc}{\ensuremath{\mt{L}}}

\newcommand{\UL}[1]{\underline{#1}}
\newcommand{\LA}{\ensuremath{\leftarrow}}

\newcommand{\clauseP}[2]{\ensuremath{\langle #1,#2 \rangle}}

\newcommand{\tree}{\ensuremath{\mathtt{T}}}
\newcommand{\dirt}{\ensuremath{\mathtt{D}}}
\newcommand{\ifaf}{iff}
\newcommand{\stt}{s.t.}
\newcommand{\wrt}{wrt.}
\newcommand{\respectively}{respectively}
\newcommand{\piP}{\omega}

\newcommand{\oF}{\pi}
\newcommand{\CC}{\mc{W}}
\newcommand{\K}{\ensuremath{\mc{K}}}
\newcommand{\pos}{\ensuremath{\mathtt{POS}}}

\begin{document}

\title{Explainable AI for Classification using Probabilistic Logic Inference}

\author{
Xiuyi Fan$^1$\and
Siyuan Liu$^1$\and
Thomas C. Henderson$^{2}$\\
\affiliations
$^1$Department of Computer Science, Swansea University, UK\\
$^2$School of Computing, University of Utah, USA\\
\emails
\{xiuyi.fan,siyuan.liu\}@swansea.ac.uk,
tch@cs.utah.edu
}

\maketitle

\begin{abstract}
The overarching goal of Explainable AI is to develop systems that
not only exhibit intelligent behaviours, but also are able to explain
their rationale and reveal insights. In explainable machine learning,
methods that produce a high level of prediction accuracy as well as
transparent explanations are valuable. In this work, we present an
explainable classification method. Our method works by first
constructing a symbolic Knowledge Base from the training data, and
then performing probabilistic inferences on such Knowledge Base with
linear programming. Our approach achieves a level of learning
performance comparable to that of traditional classifiers such as
random forests, support vector machines and neural networks. It
identifies decisive features that are responsible for a classification
as explanations and produces results similar to the ones found by
SHAP, a state of the art Shapley Value based method. Our algorithms
perform well on a range of synthetic and non-synthetic data sets.
\end{abstract}

\section{Introduction}

The need for building AI systems that are explainable has
been raised, see e.g., \cite{Doran17}. The ability to make
machine-led decision making transparent, explainable, and therefore
accountable is critical in building trustworthy systems. Producing
explanations is at the core of realising explainable
AI. Two main approaches for explainable machine learning have been
explored in the literature: (1) intrinsically interpretable methods
\cite{Rudin19}, in which prediction and explanation are both
produced by the same underlying mechanism, and (2) model-agnostic methods
\cite{molnar2019}, in which explanations are treated as a post hoc
exercise and are separated from the prediction model. In the case for
methods (1), while many intrinsically interpretable models, such as
short decision trees, linear regression, Naive Bayes, k-nearest
neighbours and decision rules \cite{Yang17} are easy
to understand, they can be weak for prediction and suffer from
performance loss in complex tasks. As for methods (2), model agnostic
approaches such as local surrogate \cite{Ribeiro16}, global surrogate
\cite{Alonso18}, feature importance \cite{Fisher18} and
symbolic Bayesian network transformation \cite{Shih18} leave the
prediction model intact and use interpretable but presumably weak
models to ``approximate'' the more sophisticated prediction
model. However, it has been argued that since model agnostic
approaches separate explanation from prediction, explanation
modules cannot be faithful representations of their prediction
counterpart \cite{Rudin19}. In this context, we present a
classification approach that produces accurate predictions and
explanations as well as supports domain knowledge incorporation.

Given a 
set of data instances, whose class membership is
known, classification is the problem of identifying to which of a set
of classes a new
instance belongs. Each instance is
characterised by a set of features $\mc{F}$. For some data $\mc{D}$, there
exists a labelling function $\labelFunc: \mc{D} \mapsto \{\pos, \neg
\pos\}$.\footnote{\pos{} stands for {\em positive}. For presentation
  simplicity, we only consider binary classification problems in this
  paper. Our approach generalises to multi-category classification by
  replacing \pos{} with class labels for each candidate class
  accordingly.} Let $D \subseteq \mc{D}$ be the training set \stt{}
for each $d \in D$, $\labelFunc(d)$ is known. For $x \in \mc{D}$, we
would like to know:

\begin{tabular}{rl}
{\bf Q1:} & {\em whether $\labelFunc(x) = \pos$;} \\
{\bf Q2:} & {\em if so, which features $f \subseteq \mc{F}$ make $\labelFunc(x) = \pos$.}
\end{tabular}

Standard supervised learning techniques answer Q1 but not Q2, which
asks for {\em decisive} features. Understanding {\em ``what causes a
  query instance $x$ to be classified as in some class $C$?''}
is as important as {\em ``does $x$ belong to $C$?''} For
instance, for a diagnostic system taking patients' medical records as
the input and producing disease classifications as the output,
pinpointing symptoms that lead to the diagnosis is as important as the
diagnosis itself.
In this paper, we propose algorithms answering both questions. In a
nutshell, we solve classification as inference on
probabilistic Knowledge Bases (KBs) learned from data. Specifically,
given training data $D$ with features $F$, we define a function
$\mc{M}$ that maps $D$ to a probabilistic KB. Then, for a query $x$,
we check whether $\mc{M}(D)$ and $x$ together entail $\pos$. Very
roughly, we take classification as evaluating
\begin{equation}
\mc{M}(D), x \models \pos. \label{eqn:mainIdea}
\end{equation}
In this way, computing explanations for $\labelFunc(x) = \pos$ in our
setting can be formulated as:
\begin{center}
{\em Given $\mc{M}(D),x \models \pos$, identify some $x' \subseteq
x$ \stt{} $\mc{M}(D),x' \models \pos$.}
\end{center}

We present two algorithms for probabilistic KB construction. The first
one constructs KBs from decision trees and the second constructs KBs
directly from data. Query classification is modelled with probabilistic
logic inference carried out with linear programming.
The main contributions are:
(i) a method of performing classification with probabilistic logic
inference; (ii) a polynomial time inference algorithm on
KBs; and (iii)
algorithms for identifying decisive features as explanations
and incorporating domain knowledge in classification and explanation.


\section{Training as Knowledge Base Construction}
\label{sec:KB}

KB construction is at the core of our approach. Specifically, a KB
contains a set of disjunction clauses and each clause has a
probability, defined formally as follows.

\begin{definition}
\label{dfn:KB}

A {\em Knowledge Base (KB)} $\{\clauseP{p_1}{c_1}, \ldots,$
$\clauseP{p_m}{c_m}\}$ is a set of pairs of clauses $c_i$
and probability of clauses $p_i = P(c_i)$, $1 \leq i \leq m$. Each
clause is a disjunction of literals and each literal is a
propositional variable or its negation.
\end{definition}
\begin{example}
  \label{exp:KB}
  With two propositional variables $\alpha$ and $\beta$,
  $\{\langle 0.6, \neg \alpha \vee \beta\rangle, \langle 0.8, \alpha
  \rangle\}$ is a simple KB containing two clauses with probabilities
  0.6 and 0.8, \respectively.
\end{example}

Generating logic clauses from data has been studied in the
literature, see e.g., \cite{Chiang01,Quinlan87} for extracting rules
from decision trees, and more recently, \cite{Mashayekhi17} for
extracting rules from random forests. Unlike these approaches where,
due to their use of strict inference methods, non-probabilistic rules
are generated, our KBs consist of probabilistic rules. Specifically,
from a decision tree constructed from the training data, we create a
clause $c$ from each path from the root to the leaf of the tree. The
probability of $c$ is the ratio between the positive samples and all
samples at the leaf. Formally, we define the KB $\K_\tree$ drawn from
a decision tree $\tree$ as follows.


\begin{definition}
\label{dfn:KB_from_tree}
  Let $\tree$ be a decision tree,
  each non-root node in $\tree$ labelled by a feature-value pair
  $a\_v$, read as feature $a$ having value $v$. Let $\{\rho_1, \ldots,
  \rho_k\}$ be the set of root-to-leaf paths in $\tree$, where each
  $\rho_i$ is of the form $\langle root, a_1\_v_1, \ldots, a_n\_v_m
  \rangle$ and $a_n\_v_m$ labels a leaf node in $\tree$. Then, the {\em KB
    drawn from $\tree$} is $\K_{\tree} = \{\clauseP{p_1}{c_1}, \ldots,
  \clauseP{p_k}{c_k}\}$ \stt{} for each $\rho_i$,
  $\clauseP{p_i}{c_i} \in \K_{\tree}$, where
  $c_i = \pos \vee \neg a_1\_v_1 \vee \ldots
  \vee \neg a_n\_v_m$, and
  $p_i$ is the ratio between positive and the
  total samples in the node labelled by $a_n\_v_m$.
\end{definition}

Algorithms~\ref{alg:path_to_clause} and \ref{alg:KB_from_tree}
construct $\K_{\tree}$ from data $D$. Specifically,
Algorithm~\ref{alg:path_to_clause} takes a root-to-leaf path from a
decision tree to generate a clause. The path with features
$a_1, \ldots, a_n$, \stt{} each feature has a value in $\{v_1, \ldots,
v_m\}$, is interpreted as
$a_1\_v_1 \wedge \ldots \wedge a_n\_v_m \rightarrow \pos,$
and read as, {\em a sample is positive if its feature $a_1$ has value
$v_1$, \ldots, feature $a_n$ has value $v_m$}. As a disjunction,
the clause is then written as
$\pos \vee \neg a_1\_v_1 \vee \ldots \vee \neg a_n\_v_m.$
Algorithm~\ref{alg:KB_from_tree} builds a tree and then constructs
clauses from paths in the tree.
Example 2 illustrates how to build a KB from a decision tree.

\begin{algorithm}[H]
\begin{small}
\caption{Clause from Tree Path}\label{alg:path_to_clause}
\begin{algorithmic}[1]
\Procedure{ClauseFromPath}{$\ti{path}$}
\State $\textit{clause} \LA$ \pos
\For {each edge \textit{e} in \textit{path}}
\State $a \LA$ feature of $e$
\State $v \LA$ value of $e$
\State $\textit{clause} \LA \textit{clause} \vee \neg a\_v$
\EndFor
\State \Return \textit{clause}
\EndProcedure
\end{algorithmic}
\end{small}
\end{algorithm}

\begin{algorithm}[H]
\begin{small}
\caption{Construct KB with Decision Tree}\label{alg:KB_from_tree}
\begin{algorithmic}[1]
\Procedure{DecsionTreeKB}{$\ti{D}$}
\State $\K_\tree \LA \{\}$; Use ID3 to compute a tree $\tree$ from
$\ti{D}$
\State $\ti{allPaths} \LA \text{all paths from the root to leaves
in } T$ \label{line:root_to_leaves}
\For {each \ti{path} in \ti{allPaths}}
\State $n \LA$ end node in \ti{path}
\State $r \LA$ ratio between positive and total samples in $n$
\State add $[r]$ \Call{ClauseFromPath}{\ti{path}} to $\K_\tree$
\EndFor
\State \Return $\K_\tree$
\EndProcedure
\end{algorithmic}
\end{small}
\end{algorithm}



\begin{example}
\label{exp:psatDT}

Given a data set with four strings, {\em 0000, 1111,
1010, 1100}, labelled positive, and four strings, {\em 0010, 0100, 1110,
1000}, labelled negative. There are four features, bits 1-4, each feature takes its value
from $\{0,1\}$. The decision tree constructed
is shown in Figure~\ref{fig:DT}. There are eight leaves, thus eight
root-to-leaf paths and clauses. E.g.,
{\ti root} $\rightarrow$ $a_4\_0$ $\rightarrow$ $a_1\_0$ $\rightarrow$
$a_2\_0 \rightarrow a_3\_0$ gives the clause
$\pos{} \vee \neg a_4\_0 \vee \neg a_1\_0 \vee \neg a_2\_0 \vee \neg
a_3\_0$. The probability of the clause is the number
of positive samples over the total samples at the leaf. There is
only one sample, 0000, at this leaf, since it is positive, the
clause probability is 1. The KB $\K_\tree$ is shown in
Table~\ref{tb:kb}.\footnote{Henceforth, $[p]$ $z_1 \vee \ldots \vee
  z_l$ denotes an $l$-literal clause in a KB with probability $p$.}

\begin{figure}[!h]
\vspace{-15pt}
  \begin{footnotesize}
\[
\xymatrix@C=10pt@R=6pt{
    &                  &                    &                    & root\ar[dl]\ar[dr]\\
    &                  &                    & a_4\_0\ar[dl]\ar[dr] &                    & a_4\_1\\
    &                  & a_1\_0\ar[dl]\ar[d]  &                    & a_1\_1\ar[d]\ar[dr] \\
    &a_2\_0\ar[dl]\ar[d] & a_2\_1               &                    & a_2\_0\ar[dl]\ar[d]   & a_2\_1\ar[d]\ar[dr] \\
a_3\_0&a_3\_1              &                    & a_3\_0               & a_3\_1                & a_3\_0                & a_3\_1 \\
}
\]
  \end{footnotesize}
\vspace{-15pt}
\caption{
 Decision tree learned from data in
 Example~\ref{exp:psatDT}.
 A node $a_X\_Y$ is read as ``bit $X$ has
 value $Y$''. \label{fig:DT}}
\end{figure}
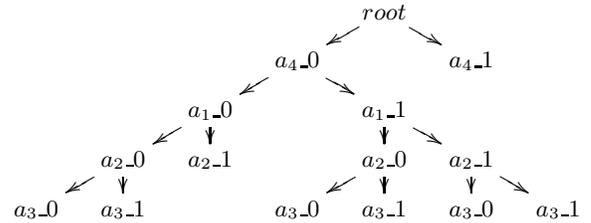

\begin{table}[!h]
\vspace{-15pt}
  \caption{$\K_\tree$ from the tree in
    Figure~\ref{fig:DT}. \label{tb:kb}}
  \begin{small}
\centerline{
\begin{tabular}{ll}
$[{\rm 0.0}]$ & \pos{} $\vee \neg a_1\_0 \vee \neg a_2\_0 \vee \neg a_3\_1 \vee \neg a_4\_0$ \\
$[{\rm 1.0}]$ & \pos{} $\vee \neg a_1\_0 \vee \neg a_2\_0 \vee \neg a_3\_0 \vee \neg a_4\_0$ \\
$[{\rm 0.0}]$ & \pos{} $\vee \neg a_1\_0 \vee \neg a_2\_1 \vee \neg a_4\_0$ \\
$[{\rm 1.0}]$ & \pos{} $\vee \neg a_1\_1 \vee \neg a_2\_0 \vee \neg a_3\_1 \vee \neg a_4\_0$ \\
$[{\rm 0.0}]$ & \pos{} $\vee \neg a_1\_1 \vee \neg a_2\_0 \vee \neg a_3\_0 \vee \neg a_4\_0$ \\
$[{\rm 0.0}]$ & \pos{} $\vee \neg a_1\_1 \vee \neg a_2\_1 \vee \neg a_3\_1 \vee \neg a_4\_0$ \\
$[{\rm 1.0}]$ & \pos{} $\vee \neg a_1\_1 \vee \neg a_2\_1 \vee \neg a_3\_0 \vee \neg a_4\_0$ \\
$[{\rm 1.0}]$ & \pos{} $\vee \neg a_4\_1$ \\
\end{tabular}
}
  \end{small}
  \vspace{-10pt}
\end{table}

\end{example}

Algorithm~\ref{alg:KB_from_tree} constructs clauses from
root-to-leaf paths in a decision tree. We can also use
paths from the root to all nodes, not just the leaves, to construct
clauses, i.e., replacing line~\ref{line:root_to_leaves} in
Algorithm~\ref{alg:KB_from_tree} with\\
\centerline{
$\textit{allPaths} \LA \text{all paths from the root to \textbf{all nodes} in } T$.
}
As random forests have been introduced
to improve the stability of decision trees, we can apply the same idea
to obtain more clauses from a forest,
i.e., repeatedly generated different decision trees, and for each tree,
we construct clauses for each path originated at its root, in the
spirit of \cite{Mashayekhi17}.
If we further take the above idea of ``generating as many clauses as
possible'' to its limit, we realise that constructing KBs from trees
is a special case of selecting clauses constructed from all
$k$-combinations of feature-value pairs, for $k = 1\ldots n$, where
$n$ is the total number of features in the
data. Formally, we define the KB $\K_\dirt$ drawn directly from
data $\dirt$ as follows.

\begin{definition}
\label{dfn:kd}
Given data $D$ with features $F = \{a_1, \ldots, a_n\}$ taking values
from $V = \{v_1, \ldots, v_m\}$, for each $F_k = \{a'_1, \ldots,
a'_k\} \in 2^F \setminus \{\}$, let $C_k^1 = \{a'_1\_v | v \in V\},
\ldots, C_k^k = \{a'_k\_v | v \in V\}$. $C_k = C_k^1 \times \ldots
\times C_k^k$. For each $c =
\{a''_{1}\_v'_{1}, \ldots, a''_{k}\_v'_{k}\} \in C_k$, $S_i \subseteq D$
is the set of samples \stt{} feature $a''_{i}$ having value $v'_{i'}$
for all $i \in \{1,\ldots,k\}$. If $|S_i| \neq 0$, then
let $p_i$ be the ratio between positive samples in $S_i$ and $|S_i|$,
$\clauseP{p_i}{\pos \vee \neg a''_{1}\_v'_{1} \vee \ldots \vee \neg
  a''_{k}\_v'_{k}}$ is in the KB $\K_{\dirt}$ {\em drawn directly from
  data}. There is no other clause in $\K_{\dirt}$ except those constructed
as above.
\end{definition}
Definition~\ref{dfn:kd} can be illustrated with the following example.
\begin{example}
Let $F = \{a_1, a_2\}$ and $V=\{0,1\}$. Then $2^F \setminus \{\} =
\{\{a_1\}, \{a_2\}, \{a_1, a_2\}\}$. For illustration, let us choose
$F_k = \{a_1,a_2\}$. Then $C_k^1 = \{a_1\_0, a_1\_1\}$, $C_k^2 =
\{a_2\_0, a_2\_1\}$, and $C_k = \{\{a_1\_0, a_2\_0\}, \{a_1\_0,
a_2\_1\},$ $\{a_1\_1, a_2\_0\}, \{a_1\_1, a_2\_1\}\}$. Then, suppose
we choose $c = \{a_1\_0, a_2\_0\}$ and add $\clauseP{p_i}{\pos \vee
  \neg a_1\_0 \vee \neg a_2\_0}$ to $\K_\dirt$, where $p_i$ is the
ratio between positive samples with both features $a_1,a_2$ having
value 0 and total samples with these feature-values. $\K_\dirt$ can be
constructed by choosing different $F_k$ and $c$ iteratively.
\end{example}

\begin{algorithm}
  \caption{Construct KB Directly}\label{alg:KB_direct}
\begin{small}
\begin{algorithmic}[1]
\Procedure{DirectKB}{$\ti{data}$}
\State $\ti{counts} \LA \{\}, \K_\dirt \LA \{\}$ \label{line:ct}
\For {each \ti{entry} in \ti{data}}
  \State $\ti{feaVals} \LA \{a\_v |$ feature $a$ has value $v$ in
  $\ti{entry}\}$
  \State \ti{label} \LA binary label of \ti{entry} as integer \label{line:label}
  \State $S \LA$ \Call{Powerset}{\ti{feaVals}} $\setminus \{\}$ \label{line:S}
  \For {each \ti{key} as an element of $S$}
    \If{\ti{key} is in \ti{counts}}
      \State $\ti{counts}[\ti{key}] \LA \ti{counts}[\ti{key}] +
  [1, \ti{label}]$ \label{line:label1}
    \Else
      \State $\ti{counts}[\ti{key}] \LA
  [1, \ti{label}]$ \label{line:label2}
    \EndIf
  \EndFor
\EndFor
\For {each \ti{key} in \ti{counts}}
  \State $r \LA \ti{counts}[\ti{key}][1] / \ti{counts}[\ti{key}][0]$
  \State Insert ``$[r]$ \ $\pos \vee \neg key$''
  to $\K_\dirt$ \label{line:negKey}
\EndFor
\State \Return $\K_\dirt$
\EndProcedure
\end{algorithmic}
\end{small}
\end{algorithm}

Algorithm~\ref{alg:KB_direct} gives a procedural construction for
$\K_{\dirt}$.\footnote{In Line~\ref{line:negKey}, $\neg \{s_1, \ldots,
  s_n\}$ is $\neg s_1 \vee \ldots \vee \neg s_n$, e.g. for
  $\ti{key}=\{a_1\_v_1,a_2\_v_2\}$, insert ``$[p]$ $\pos \vee \neg
  a_1\_v_1 \vee \neg a_2\_v_2$'' to $\K_\dirt$.
$\ti{counts}$ is a dictionary with keys being sets of
feature-value pairs and values being two-element arrays. \ti{label}
is either 0 or 1. Line~\ref{line:label1} is an element-wise
addition, e.g., [1,0]+[1,1]=[2,1]. At the end of the first loop,
$\ti{counts}[key][0]$ is the number of samples containing \ti{key} and
$\ti{counts}[key][1]$ is the number of positive ones.}
The following propositions describe the relation between the two KB
construction approaches. Proposition~\ref{prop:subset} and
\ref{prop:pathInTree} sanction that
all clauses extracted from decision trees can be constructed directly
in $\K_\dirt$ and all
clauses built in $\K_\dirt$ can be extracted from some trees, \respectively.

\begin{proposition}
\label{prop:subset}
Given a data set $D$,
$\K_\tree \subseteq \K_\dirt$.
\end{proposition}
\begin{proof}
(Sketch.) $S$ constructed in Line~\ref{line:S},
Algorithm~\ref{alg:KB_direct} is the powerset of all possible
feature-value pairs in $D$ and a path in a decision tree
represents some feature-value pairs in $D$. Thus, any clause
produced by a tree is produced by Algorithm~\ref{alg:KB_direct}.
\end{proof}

\begin{proposition}
\label{prop:pathInTree}
Given a data set $D$, for each clause $c
\in \K_\dirt$, there exists a decision tree $\tree$ constructed from
$D$ \stt{} there is a path $p$ in $\tree$ and the clause drawn from
$p$ is $c$.
\end{proposition}
\begin{proof}
(Sketch.) All clauses in $\K_\dirt$ are of the form $\pos \vee \neg
a_1\_v_1 \vee \ldots \vee \neg a_n\_v_m$ where $a_i\_v_j$ are
feature-value pairs and for any $a_i, a_j \in \{a_1, \ldots, a_n\}$,
if $i \neq j$, then $a_i \neq a_j$. Thus, one can construct a tree
$\tree$ containing the path \ti{root} \!\! $\rightarrow$ \!\! $a_1\_v_1$
\!\! $\rightarrow$ $\ldots$ $\rightarrow$ \!\! $a_n\_v_m$.
\end{proof}

\section{Querying as Probabilistic Inference}
\label{sec:inf}

Our KB construction methods produce clauses with
probabilities. Intuitively, for a query that asserting some
feature-value pairs, we want to compute the probability of $\pos$ under
these feature-value pairs and predicting the query being positive when
the probability is greater than 0.5. To
introduce our inference method for computing such probabilities, we
first review a few concepts in probabilistic logic \cite{Nilsson86},
which pave the way for discussion.

Given a KB $\K$\footnote{From this point on, we use $\K$
to denote a KB constructed using either of the two approaches
($\K_\tree$ or $\K_\dirt$).} with clauses $c_1, \ldots, c_m$ composed
from $n$ propositional variables, the {\em complete conjunction set},
as
\CC, over $\K$ is the set of $2^n$ conjunctions \stt{} each
conjunction contains $n$ distinct propositional variables.
A {\em probability distribution} $\oF$ (\wrt{} \K) is the set of $2^n$
probabilities $\oF(w) \geq 0, (w \in \CC)$ \stt{}
$\sum_{w\in\CC}{\oF(w)}=1$. $\oF$ {\em satisfies} \K{} \ifaf{} for
each $i=1,\ldots,m$, the sum of $\oF(w)$ equals $P(c_i)$ for all $w$
\stt{} the truth assignment satisfying $w$ satisfies $c_i$. A KB \K{}
is {\em consistent} \ifaf{} there exists a $\oF$
satisfying \K.

With a consistent KB, Nilsson suggested that one can derive {\em
literal probabilities} from $\oF$, i.e., for all
literals $z$ in the KB, $P(z)$ is the sum of $\oF(w)$ for all
$w \in \CC$ containing $z$, e.g., for a consistent KB with two
literals $\alpha$ and $\beta$, $P(\alpha) = P(\alpha \wedge \beta) +
P(\alpha \wedge \neg \beta)$ \cite{Nilsson86}. In short, to compute
literal probabilities, one first computes probability assignments over
the complete conjunction set, and then adds up all relevant
probabilities for the literal.

At first glance, since $\pos$ is an literal in our knowledge
base, it might be possible to perform our inference with the above
approach for computing $P(\pos)$: all clauses in a KB are of
the form $\pos \vee \neg a_1\_v_1 \vee \ldots \vee \neg a_n\_v_m,$
each with an associated probability; a query is a set of
feature-value pairs, e.g., $a_1'\_v_1', \ldots, a_n'\_v_m'$, each with
an assigned probability 1; $P(\pos)$ computed as the sum of $P(\pos
\wedge a_1\_v_1 \wedge \ldots \wedge a_n\_v_m)$, $P(\pos
\wedge a_1\_v_1 \wedge \ldots \wedge \neg a_n\_v_m)$, \ldots, $P(\pos
\wedge \neg a_1\_v_1 \wedge \ldots \wedge \neg a_n\_v_m)$ estimates
the likelihood of $\pos$.
However, this idea fails for the following two reasons. Firstly, this approach requires solving the probability distribution
$\oF$, which has been shown to be NP-hard \wrt{} the number of
literals in the KB\cite{Georgakopoulos1988}, thus the
state-of-the-art approaches only work for KB
with a few hundred of variables \cite{Finger11}.

Secondly, putting a KB and a query together introduces
inconsistency, so there is no solution for $\oF$. For
instance, for the KB in Example~\ref{exp:psatDT}, let the query be
0000, which translates to  four clauses, $a_1\_0, a_2\_0, a_3\_0$ and
$a_4\_0$, each with $P(a_i\_0)=1$. Consequently, $P(\neg
a_i\_0)=0$. Together with $P(\pos{} \vee \neg a_1\_0 \vee \neg
a_2\_0 \vee \neg a_3\_0 \vee \neg a_4\_0)=1$, we infer
$P(\pos)=1$. However, $P(\pos)=1$ is inconsistent with
$P(\pos{} \vee \neg a_1\_0 \vee \neg a_2\_0 \vee \neg a_3\_1 \vee \neg
a_4\_0) = 0$, as for any $\alpha, \beta$, we must have $P(\alpha) \leq
P(\alpha \vee \beta)$. In this case, \K{} is inconsistent with
the query thus there is no solution for $\oF$.

One might suspect the inconsistency illustrate above is an artefact of
our KB construction, i.e., there could exist ways to construct KB
\stt{} consistency can be ensured. Although this might be
the case, there is no such existing method as far as we know and
when we incorporate domain knowledge later in this paper, it
becomes clear that being able to tolerate inconsistency is useful.

Since the source of the complexity is in the computation of the
probability distribution over the complete conjunction set, we avoid
computing it explicitly and introduce an efficient algorithm for
estimating literal probabilities without computing $\oF$. We formulate
the computation as an optimization problem so that inconsistency is
tolerated. This is the core of our inference method.

\begin{definition}
\label{dfn:lp}

  Given a KB $\K{} =
  \{\clauseP{p_1}{c_1},\ldots,\clauseP{p_m}{c_m}\}$ with
  clauses $\mc{C} = \{c_1, \ldots, c_m\}$
  over literals $\mc{Z}$, a {\em linear program} $L_{\K}$ of \K{}
  with unknowns $\piP(\sigma), \sigma \in \mc{C} \cup \mc{Z}$, is
  the following.\\
  {\rm minimise:}
  \begin{equation}
     \sum_{i = 1}^{m}|\piP(c_i) - p_i|  \label{eqn:obj}
  \end{equation}

  {\rm subject to: for each clause} $c_i = z_1 \vee \ldots \vee z_l$,
\begin{equation}
\piP(c_i) \leq \piP(z_1) + \ldots + \piP(z_l)\label{eqn:bool};
\end{equation}
\hspace{52pt} {\rm for} $z_j = z_1 \ldots z_l$ {\rm in clause} $c_i$:
\begin{align}
\piP(c_i) &\geq \piP(z_{j}); \label{eqn:mono}\\
1 &= \piP(z_{j}) + \piP(\neg z_{j}); \label{eqn:one}\\
0 &\leq\piP(z_{j}) \leq 1. \label{eqn:gt0a}
\end{align}
\end{definition}

Definition~\ref{dfn:lp} estimates literal probabilities
from clause probabilities without computing the
distribution over the complete conjunction set, i.e., for
any literal $z$ in the KB, $\piP(z)$ approximates $P(z)$.
The intuition is as follows.
\begin{itemize}[noitemsep]
\item
Constraints given by Eqn.~(\ref{eqn:bool}-\ref{eqn:gt0a}) are
probability laws, i.e., Eqn.~(\ref{eqn:bool}) is the Boole's
inequality \cite{casella2002}, (\ref{eqn:mono}) is monotonicity;
(\ref{eqn:one}) and (\ref{eqn:gt0a}) define the bound.
\item
The optimisation function Eqn.~(\ref{eqn:obj}) is used to tolerate
inconsistency, i.e., for a KB containing inconsistent clauses, \stt{}
some of the constraints cannot be met, we allow clause probabilities
to be relaxed by not forcing $\piP(c_i) = P(c_i)$ as constraints. We
still want the estimated clause probabilities ($\piP(c_i)$) to be as
close to their specified values ($P(c_i)$) as possible, so
Eqn.~(\ref{eqn:obj}) minimises their difference. A linear
difference is chosen to ensure a low computational complexity.
\end{itemize}
Note that, for all literals in a clause, their estimated probabilities
are constrained by inequalities local to the clause (e.g.,
$\piP(c_i) \leq \piP(z_1) + \ldots + \piP(z_l)$). We avoid the
exponential growth of constraints, which causes the NP
computational difficulties, by forgoing not only explicit probability
computation for the complete conjunction set but also global
constraints on estimated clause probabilities, e.g., for two clauses
$c_1 = \alpha \vee \beta$ and $c_2 = \alpha \vee \beta \vee \gamma$,
we do not enforce $\piP(c_1) \leq \piP(c_2)$.
We illustrate probability computation with the
following example.


\begin{example}
\label{exp:psat}
(Example~\ref{exp:KB} cont.)
Given these two clauses,
$c_1 = \neg \alpha \vee \beta; c_2 = \alpha$, and their probabilities,
$P(c_1)= 0.6, P(c_2)=0.8$,
the complete
conjunction set
$\CC = \{\neg \alpha \wedge \neg \beta, \neg
\alpha \wedge \beta, \alpha \wedge \neg \beta, \alpha \wedge
\beta\}$. Truth assignments satisfying
$\alpha \wedge \beta, \neg \alpha \wedge \beta$, and
$\neg \alpha \wedge \neg \beta$ satisfy
 $c_1$ and truth assignments satisfying $\alpha \wedge \beta$ and
 $\alpha \wedge \neg \beta$ satisfy $c_2$.
\K{} is consistent \ifaf{}
$\oF_1 = \oF(\alpha \wedge \beta)$,
$\oF_2 = \oF(\alpha \wedge \neg \beta)$,
$\oF_3 = \oF(\neg \alpha \wedge \beta)$, and
$\oF_4 = \oF(\neg \alpha \wedge \neg \beta)$ \stt{}
$\sum_{j=1}^{4}{\oF_j}=1$, $\oF_1 + \oF_3 + \oF_4 = 0.6$ and
$\oF_1 + \oF_2 = 0.8$. $L_\K$ is:

{\rm minimise:} 

\centerline{
$|\piP(c_1)-0.6| + |\piP(c_2)-0.8|$
}

{\rm subject to:}

\begin{tabular}{cccc}
\multicolumn{2}{l}
{$\piP(c_1) \leq \piP(\neg \alpha) + \piP(\beta);$} &
\multicolumn{2}{l}
{$\piP(c_2) \leq \piP(\alpha);$} \\
\multicolumn{4}{l}
{$\piP(c_1) \geq \piP(\neg \alpha);$ \hspace{12pt}
$\piP(c_1) \geq \piP(\beta);$ \hspace{12pt}
$\piP(c_2) \geq \piP(\alpha);$} \\
\multicolumn{2}{l}
{$1 = \piP(\alpha) + \piP(\neg \alpha);$}&
\multicolumn{2}{l}
{$1 = \piP(\beta) + \piP(\neg \beta);$}\\
\multicolumn{2}{l}
{$0 \leq \piP(\alpha) \leq 1;$}& 
\multicolumn{2}{l}
{$0 \leq \piP(\beta) \leq 1.$}
\end{tabular}

\noindent
A solution to $L_\K$ is:
$\piP(\neg \alpha \vee \beta) = 0.6;$
$\piP(\alpha) = 0.8;$
$\piP(\neg\alpha) = 0.2;$
$\piP(\beta) = 0.6;$
$\piP(\neg\beta) = 0.4.$

It is easy to see that \K{} is consistent, and for all literals $z$ in
\K, $\piP(z)$ is a probability assignment for $z$.
Definition~\ref{dfn:lp} gives a means of performing probabilistic
inference, as this Example can be seen as modus ponens, i.e., from
$(\alpha \rightarrow \beta, \alpha) \vdash \beta$ where
$P(\alpha \rightarrow \beta)=0.6$, $P(\alpha)=0.8$, we infer
$\piP(\beta) = 0.6$.
\end{example}

In general, for a literal $z$ in a KB $\K$, it may be the case that no
$\oF{}$ exists such that $\piP(z)$ equals the probability computed
from $\oF$. E.g., consider:
\begin{example}
  Let $\K = \{c_1 \ldots c_4\}$, in which

\centerline{
\setlength{\tabcolsep}{2pt}
\begin{tabular}{llll}
$c_1$ is $[{\rm 1.0}]$ & $\alpha \vee \beta;$ \hspace{40pt}&
$c_2$ is $[{\rm 1.0}]$ & $\alpha \vee \gamma;$ \\
$c_3$ is $[{\rm 1.0}]$ & $\beta \vee \gamma;$ \hspace{40pt} &
$c_4$ is $[{\rm 1.0}]$ & $\alpha \vee \beta \vee \gamma.$
\end{tabular}
}

\noindent
Then, $\piP(c_i) = 1$,
$\piP(\alpha) = \piP(\beta) = \piP(\gamma) = 0.5$ is a solution to
$L_\K$ where the objective function attains 0. However,
$\piP(z) \neq P(z)$, for $z = \alpha, \beta, \gamma$.\footnote{This
shows that $L_\K$ has a feasible region larger than
the solution space of $\oF$. However, since linear programming
algorithms look for solutions at the boundary of variables, we do not
see such solutions in practice. Indeed, the Gurobi solver finds
$\piP(z) = P(z) = 1$, for $z = \alpha, \beta, \gamma$, which {\em are}
in the solutions computed with Nilsson's method.}
\end{example}

Relations between literal probability found via computing
exact solutions from the distribution over the complete conjunction
set and solutions found in $L_\K$ are as follows.

\begin{lemma}
\label{lemma:linearSystem}

If $\K$ is consistent, then solutions for
all $\piP(z)$ $z \in \mc{Z}$ exist \stt{} Eqn. (\ref{eqn:obj})
minimises to 0.
\end{lemma}
\begin{proof}
(Sketch.) Eqn. (\ref{eqn:obj}) minimises to 0 only when $\piP(c_i)
= p_i$ for all $c_i$. If \K{} is consistent, then there is an
assignment of values to the literal probabilities that satisfies the
constraints, 
i.e., for all literals $z$ and all clauses $c$, $\piP(z)\!\!=
\!\!P(z)$ and $\piP(c)\!\! = \!\!P(c)$ minimise (\ref{eqn:obj}) to 0.
\end{proof}

\begin{corollary}
\label{corollary:linearSystem}

Given a KB \K, if $L_\K$ does not minimise to 0, then $\K$ is not
consistent.
\end{corollary}
\begin{proof}
(Sketch.) By the contrapositive of Lemma~\ref{lemma:linearSystem},
if there is no assignment that solves the linear
  programming problem, then there is no exact solution for $\oF$.
\end{proof}

\begin{proposition}
\label{prop:ployTime}

Given a KB $\K$ with $n$ propositional variables $x_1,\ldots, x_n$,
each $\piP(x_i)$ in $L_\K$ can be computed in polynomial time \wrt{}
$n$.
\end{proposition}
\begin{proof}
(Sketch.) Let $u$ be the number of unknowns in $L_\K$, $m$ the number
of clauses in $\K$, $u = 2n + m$. Linear Programming is polynomial
time solvable.
\end{proof}

It is theoretically interesting to ask, for
consistent KBs, what the error bound between literal probability
computed with Nilsson's method and our linear programming
method is, subject to a chosen linear programming solver. However, in the
context of this work, answering such question is less important as KBs
generated by our approach are not necessarily consistent. For such
KBs, Nilsson's approach gives no solution thus these is no ``error
bound'' exists.

With a means to reason with KBs, we are ready to answer queries.
Algorithm~\ref{alg:KB_to_query} defines the query process.
Let $L_\K$ be the linear system constructed from $\K$.
Given a query $\mc{Q}$ with feature-value pairs
$a_1\_v_1,$ $\ldots,a_n\_v_m$, we amend $L_\K$ by inserting
$\piP(a_i\_v_j) = 1$ and $\piP(a_i\_v_j') = 0$, where $v_j'$ is a
possible value of $a_i$, $v_j' \neq v_j$, for all $a_i, v_j$ in
$\mc{Q}$. $\piP(\pos)$ computed in $L_\K$ answers whether $\mc{Q}$ is
positive. Since the solution of $\piP(\pos)$ can be a range, we
compute the upper and lower bounds of $\piP(\pos)$ by maximising and
minimising $\piP(\pos)$ subject to minimising Eqn.(\ref{eqn:obj}), respectively, and use the average of the two. It returns {\em
positive} when the average is greater than 0.5.
The intuition of our approach is that,
for a query $x$, to evaluate whether $\K, x
\models \pos$, we compute $\piP(\pos)$ in $L_\K$, in which
$\K$ is treated ``defeasibly'' \stt{} the probabilities of a clauses
in $\K$ can be relaxed
whereas the query $x$ is treated ``strictly''
as constraints in $L_\K$. Example 6 illustrates the query process.
\begin{algorithm}
  \caption{Query Knowledge Base}\label{alg:KB_to_query}
\begin{small}
\begin{algorithmic}[1]
\Procedure{QueryKB}{\ti{query}, $L_\K$}
\For {each feature \textit{a} in \textit{query}}
\For {each possible value \textit{v} of \textit{a}}
\If{\textit{a} has value \textit{v} in \textit{query}}
\State Add \textit{$\piP(a\_v) = 1$} to $L_\K$
\Else
\State Add \textit{$\piP(a\_v) = 0$} to $L_\K$
\EndIf
\EndFor
\EndFor
\State \Return $\piP(\pos)$ computed in $L_{\K}$
\EndProcedure
\end{algorithmic}
\end{small}
\end{algorithm}

\begin{example}
\label{exp:query}

(Example~\ref{exp:psatDT} cont.) For query {\rm 0101},
we add the following equations as constraints to $L_\K$:\\
\centerline{
\setlength{\tabcolsep}{4pt}
\begin{tabular}{llll}
$\piP(a_1\_0) = 1$, & $\piP(a_1\_1) = 0$, &
$\piP(a_2\_0) = 0$, & $\piP(a_2\_1) = 1$, \\
$\piP(a_3\_0) = 1$, & $\piP(a_3\_1) = 0$, &
$\piP(a_4\_0) = 0$, & $\piP(a_4\_1) = 1$.
\end{tabular}
}
The computed $\piP(\pos)$ is no greater than 0.5, representing a
negative classification.
\end{example}

The proposed querying mechanism differs fundamentally from
that of decision trees. A decision tree query can
be viewed as finding the longest clause in the KB that matches with
the query in and checking whether its probability is greater
than 0.5. For instance, for query {\em 0101}, a decision tree
query returns positive as the longest matching clause in
``\pos{} $\vee \neg a_4\_1$'' has probability 1. However, our approach
considers probabilities from other clauses in the $\K$ and produces a
different answer.
%

Since KB constructed with Algorithm~\ref{alg:KB_direct} contains
far more clauses than Algorithm~\ref{alg:KB_from_tree},
to improve
query efficiency, for a given query $Q$, we can construct a KB that
only contains clauses {\em directly relevant} to $Q$, as shown in Example 7, and perform query
on this subset of clauses, as shown in
Algorithm~\ref{alg:direct_query}.\footnote{In
line~\ref{line:contains}, a clause containing a {\em key} is defined
syntactically, e.g., ``$\pos \vee \neg a_1\_0 \vee \neg a_2\_0$''
contains $\{\neg a_1\_0, \neg a_2\_0\}$.} Query performed on the
relevant KB gives the same result as in the full KB, $\K_\dirt$, as
irrelvant clauses give no additional constraint to $\piP(\pos)$.

\begin{algorithm}
  \caption{Construct Relevant Knowledge Base}\label{alg:direct_query}
\begin{small}
\begin{algorithmic}[1]
\Procedure{QueryRelevant}{$Q, \K_\dirt$}
\State $\ti{feaVals} \LA \{a\_v |$ feature $a$ has value $v$ in
$\ti{Q}\}$
\State $S \LA$ \Call{Powerset}{\ti{feaVals}} $\setminus \{\}$,
$\ti{relevantKB} \LA \{\}$
\For {each \ti{key} an element of $S$}
  \For {each \ti{clause} in $\K_\dirt$}
    \If {\ti{clause} contains \ti{key}} \label{line:contains}
    \State Insert \ti{clause} to \ti{relevantKB}
    \EndIf
  \EndFor
\EndFor
\State \Return \ti{relevantKB}
\EndProcedure
\end{algorithmic}
\end{small}
\end{algorithm}

\begin{example}
\label{exp:rel}

(Example~\ref{exp:query} cont.) {\em relevantKB} for query 0101 is
follows:

\hspace{-15pt}
\begin{small}
\begin{tabular}{llll}
  $[0.33]$ & \pos{} $\vee \neg a_1\_0$ &
  $[0.5]$  & \pos{} $\vee \neg a_2\_1$ \\
  $[0.5]$ & \pos{} $\vee \neg a_3\_0$ &
  $[1.0]$ & \pos{} $\vee \neg a_4\_1$ \\
  $[0.0]$ & \pos{} $\vee \neg a_1\_0 \vee \neg a_2\_1$ &
  $[0.5]$ & \pos{} $\vee \neg a_1\_0 \vee \neg a_3\_0$ \\
  $[0.5]$ & \pos{} $\vee \neg a_2\_1 \vee \neg a_3\_0$ &
  $[1.0]$ & \pos{} $\vee \neg a_2\_1 \vee \neg a_4\_1$ \\
  $[0.0]$ &
  \multicolumn{3}{l}
  {\pos{} $\vee \neg a_1\_0 \vee \neg a_2\_1 \vee \neg a_3\_0$}
\end{tabular}
\end{small}
\end{example}

Overall, our method is non-parametric so no tuning is required.
Query generalization is the result of restricting the
solution space of
$\piP(\pos)$ through clauses describing subsets of the query. In
Example~\ref{exp:rel}, {\em 0101} is not in the training
set. However, the relations between its substrings and $\pos$ are
described by clauses in the KB.
Jointly, these clauses decide $\piP(\pos)$, which approximates
$P(\pos)$ for this query.

\section{Explanation and Knowledge Incorporation}
\label{sec:expl}

Several methods for comparing feature importance as a form of
explanation have been introduced in the literature. Some of these
methods, e.g. \cite{Zhao19} and \cite{Apley16}, study the relation
between features and the overall
classification for all training cases. They are ``global'' methods in
the sense that they answer the question: {\em ``Which
feature has the strongest correlation with the class label in a
dataset?''} Whereas other methods, notably Shaply Value based
approaches \cite{Strumbelj14,Lundberg17,shapley}, study feature value
contribution for individual instances. They are ``local'' and answer:
{\em ``For a given query instance, how much contribution does each of
  its feature value make?''} In this sense, ours is a local approach
that explains query instances.

One advantage of the presented classification method is that it
supports {\em partial queries}, which are queries with missing
values, as the probability of $\pos$ can be computed without values
assigned to all features. Explanation computation can be supported
with partial queries in our approach.
%
%
Algorithm~\ref{alg:compExp} outlines one approach. Given a query $Q$
with $n$ features, to find the $k$ most decisive features, we
construct {\em sub-queries} \stt{} each sub-query contains exactly $k$
feature-value pairs in $Q$. If $Q$ yields a positive classification,
then the sub-query that maximises $\piP(\pos)$ is an explanation;
otherwise, the sub-query that minimises $\piP(\pos)$ is. Since we know
that there are $\tbinom{n}{k}$ different sub-queries in total, the
order of sub-query evaluation can be strategised with methods such as
hill climbing for more efficient calculation.
Although in principle, Algorithm~\ref{alg:compExp} could
work with any classification technique supporting partial
queries, our proposed method does not require reconstructing the
trained model for testing each of the sub-queries, making the
explanation generation convenient. The explanation approach is illustrated in Example~\ref{exp:expl}.

\begin{algorithm}[H]
  \caption{Explanation Computation}\label{alg:compExp}
\begin{small}
\begin{algorithmic}[1]
\Procedure{ComputeExplanation}{$Q, L_\K, k$}
\State $S \LA \{sQ|sQ \in 2^{\textit{Q}}, \Call{SizeOf}{sQ} = k\}$
\If{$\Call{QueryKB}{Q,L_\K} > 0.5$}
\State \Return $\argmax_{sQ \in S} \Call{QueryKB}{sQ,L_\K}$
\Else
\State \Return $\argmin_{sQ \in S} \Call{QueryKB}{sQ,L_\K}$
\EndIf
\EndProcedure
\end{algorithmic}
\end{small}
\end{algorithm}

\begin{example}
\label{exp:expl}

(Example~\ref{exp:query} cont.) To compute the single most decisive
feature, we let $k = 1$. $S$ contains four feature-value pairs:
%
$q_1 = \{a_1\_0\},$ 
$q_2 = \{a_2\_1\},$ 
$q_3 = \{a_3\_0\},$ 
$q_4 = \{a_4\_1\}.$
%
Let $\piP_i, i=1 \ldots 4$ be $\piP(\pos)$ computed with $q_1 \ldots
q_4$, \respectively.
We have $\piP_1 = 0.33, \piP_2 = 0.5$,
$\piP_3 = 0.5$, and $\piP_4 = 1$. Thus, the computed explanation for
the classification is $a_1\_0$. We read this as:
\centerline{\em 0 - - - is responsible for 0101 being negative.}
This matches with our intuition well as for each of the other
choices, there are at least as many positive samples as negative
ones.

\end{example}

Note that there is a subtle difference between our
approach and Shaply Value based methods. Upon computing a $k$-feature
explanation, our approach considers $\tbinom{n}{k}$ $k$-feature
coalitions and select the ``most decisive'' coalition. Wherease Shaply
Value approaches consider each feature individually and returns the
set of $k$ most decisive fetures.

Incorporating domain knowledge to complement data-driven
machine learning is supported by our approach. Since a KB
consists of probabilistic clauses, any knowledge $\K'$, about either a
specific query or the overall model, can be used alongside $\K$, as
long as it is represented in clausal form. In other words,
Equation~\ref{eqn:mainIdea} can be revised to
\begin{equation}
\mc{M}(D), \mc{K}', x \models \pos. \label{eqn:mainIdea2}
\end{equation}
Two advantages of our approaches are (1) incorporated
knowledge is used in the same way as clauses learned from data; and
(2) since the inference process tolerates inconsistency, {\em
  incomplete} or {\em imperfect} knowledge can be incorporated. For
instance, suppose we somehow know it is ``mostly true'' that {\em a
  string is positive if either its 3rd or 4th digit is 0}. If we
liberally take ``mostly true'' as, saying, probability 0.9, this can be
represented as $a_3\_0 \vee a_4\_0
\rightarrow \pos$, so we insert

\centerline{
\begin{tabular}{llllllll}
$[0.9]$ & $\pos \vee \neg a_3\_0$ &&&&& $[0.9]$ & $\pos \vee \neg a_4\_0$
\end{tabular}
}
\noindent
into $\K$ to complement clauses learned from data. Although similar
clauses or even the same clause with different probabilities may
already exist in the KB, our ability of tolerating inconsistencies
could accommondate such knowledge, as shown in the next section.

\section{Performance Analysis}
\label{sec:perf}

Definition~\ref{dfn:lp} gives an efficient system construction.
As shown in Figure~\ref{fig:prob},
we can solve KBs containing up to 10,000 variables and
10,000 clauses within a few seconds on a single CPU workstation with
an Xeon 2660v2 processor and 32GB RAM.
The ability of approaching
KBs of such large sizes enables solving practical classification
tasks.

\begin{figure}[!htb]
\centerline{
\includegraphics[trim={0.4cm 0cm 0.3cm 0.0cm},
                       clip,width=0.53\textwidth]{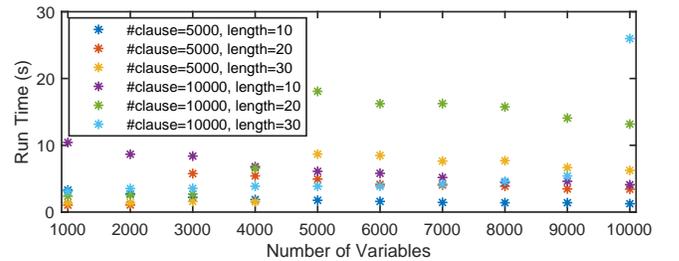}}
\vspace{-3pt}
\caption{Experiment results from KB with different
  sizes. \label{fig:prob}}
\end{figure}

To evaluate the proposed classifiers, we first conduct experiments on six real data sets, with results shown in Table~\ref{table:perfEvalReal}.
\begin{table}[!htb]
\caption{Experiment results (F$_1$ scores) with multiple data sets and
  several baseline algorithms. \label{table:perfEvalReal}}
\begin{scriptsize}
\centerline{
\begin{tabular}{|r|cccccc|}
\hline
       & Titanic & Mushroom & Nursery & HIV-1    & Bill & Vehicle \\
\hline
Tree    & 0.79     & {\bf0.99} & 0.99     & 0.87      & 0.98      & 0.95 \\
Direct & 0.79     & {\bf0.99} & 0.99     & 0.97      & {\bf0.99} & 0.96 \\
\hline
CART   & {\bf0.82}& {\bf0.99} & 0.99     & 0.94      & {\bf0.99} & {\bf0.98} \\
MLP    & 0.81     & {\bf0.99} & 0.99     & 0.73      & 0.98      & 0.96 \\
Forest & {\bf0.82}& {\bf0.99} & {\bf 1}  & 0.98      & {\bf0.99} & {\bf0.98} \\
SVM    & 0.78     & {\bf0.99} & 0.99     & {\bf0.99} & {\bf0.99} & 0.97 \\
\hline
\end{tabular}}
\end{scriptsize}
\vspace{-15pt}
\end{table}
For each data set, we measure the performance with the F$_1$ score,
taken as the average of 50 runs for each data set. Our
approaches are {\em Tree} (Algorithm 2) and {\em Direct} (Algorithm 3). We use CART (a decision
tree algorithm),  multi-layer perceptron (MLP) neural networks (with
two hidden layers with 12 and 10 nodes, \respectively), random forest
(with 100 trees) and support vector machine as our comparison
baselines. The six real data sets include the
Titanic \footnote{\url{https://www.kaggle.com/c/titanic}},
Mushroom, Nursery and HIV-1 protease cleavage data sets from the
UCI Machine Learning Repository \cite{Dua2017}, the UK parliament bill
data set reported in \cite{Cyras19} as well as an image data set for
vehicle classification. For the Titanic data set, we used seven
discrete
features -- ticket class, sex, age (discretized to 4 categories), number
of siblings, number of parents, passenger fare (discretized to 3
categories), and port of embarkation. For the Mushroom data set, we
used the first 11 features. For the multi-class data set Nursery, we
randomly selected two classes and discarded others. For the Parliament
bill, we used five features -- House of Commons or House of Lords, type
of bill, number of sponsors, bill subject, and final stages of the
bill. The vehicle image data set contains 1635 images with 767 of them
being cars and the rest busses and trucks. Feature extraction has been
applied with 12 features created for each image. They are: number of
pixels of the object, shape coefficient 1-5, mean and standard
deviation of RGB channels.  Each data set has been pre-processed such
that the positive and the negative samples are balanced by randomly
replicating samples in the smaller class. For all data set, the ratio
between training and testing is 70\% to 30\%. Overall, we see that
{\em Direct} gives satisfactory performance.

To evaluate our explanation approach, we first compare {\em Direct}
with the state of the art Shapley Value based approach SHAP
\cite{shapley},
using the Titantic and Mushroom data sets. The results are shown in
Figure~\ref{fig:comp1}, with Figure~\ref{fig:comp1}(a)(b) showing the results from the Titantic data set and Figure~\ref{fig:comp1}(c)(d) from the Mushroom.
Figure~\ref{fig:comp1}(a) shows the percentage of the same features
suggested as explanations for different explanation lengths (i.e., $k=1,2,3,4,5$). For example,
when $k=1$ (computing one-feature explanations), 75\% of
all instances have the same feature chosen as
the explanation by both approaches. When $k=2$ (computing two-feature
explanations),
there are 72\% and 25\% instances found with the same 1 and 2
features, respectively. Figure~\ref{fig:comp1}(b) shows the percentage
of each feature being selected as an explanation across all
instances. We see that when $k=1$, ours and SHAP both suggest that
feature 2
explains the classification result for over 70\%
instances. When $k=2$, the two approaches agree that feature 2 is an
explanation while differing on the choice for the other feature.
\begin{figure*}[!h]
\centering
\begin{tabular}{@{}c@{}}
\includegraphics[trim={3cm 1cm 3cm 1cm},
                       clip,width=0.95\textwidth]{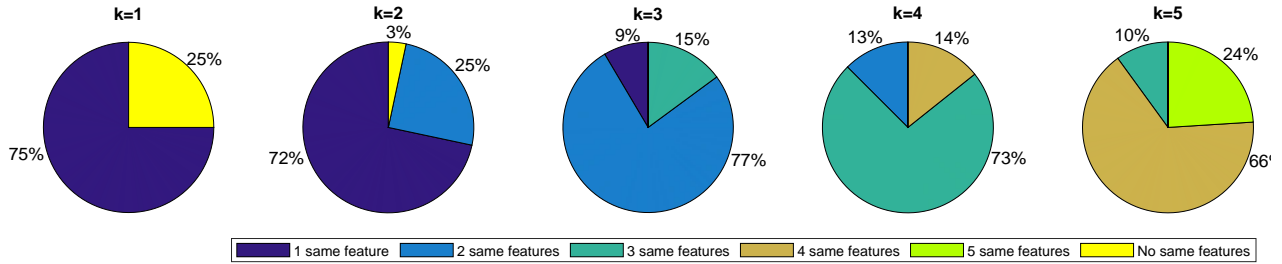}\\
(a) The percentages of the same explanations suggested by Direct and Shapley over the Titantic data\\
\includegraphics[trim={3cm 1.5cm 3cm 0cm},
                       clip,width=0.95\textwidth]{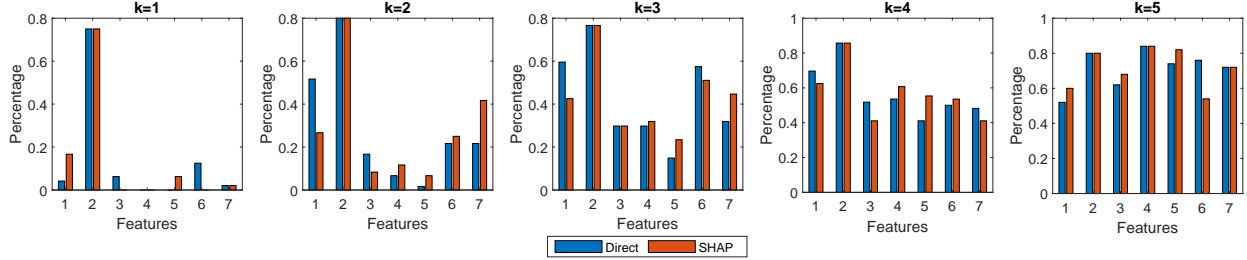}\\
(b) The percentages of features serving as explanations suggested by Direct and Shapley over the Titantic data set\\
\includegraphics[trim={3cm 1cm 3cm 1cm},
                       clip,width=0.95\textwidth]{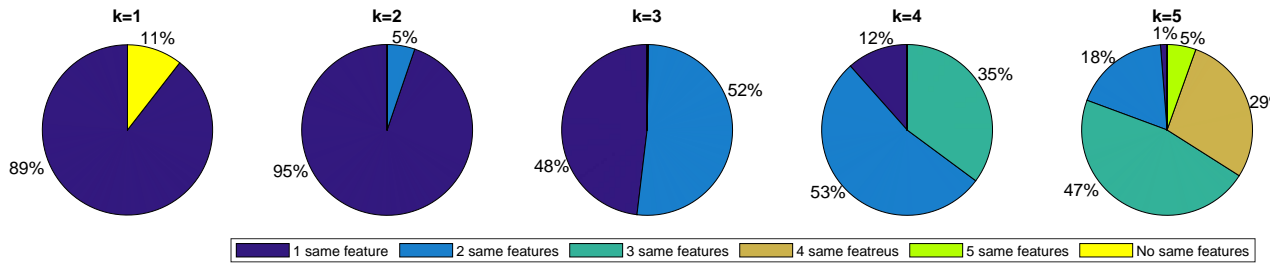}\\
(c) The percentages of the same explanations suggested by Direct and Shapley over the Mushroom data set\\
\includegraphics[trim={3cm 1.5cm 3cm 0cm},
                       clip,width=0.95\textwidth]{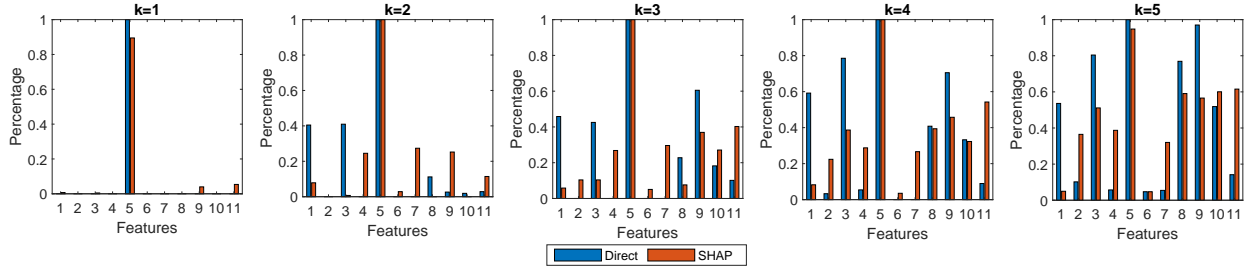}\\
(d) The percentages of the same explanations suggested by Direct and Shapley over the Mushroom data set\\
\end{tabular}
\caption{Explanation results comparison.} \label{fig:comp1}
\vspace{-10pt}
\end{figure*}

Results presented in Figure~\ref{fig:comp1} shows that our approach
gives similar results to SHAP. As there is no
explanation ground
truth in these data sets, it is impossible to decide who gives
``correct'' explanations. To address this, we performed further
experiments with synthetic data sets with known explanation ground
truth. Specifically, we created four synthetic data sets of integer
strings, Syn 10/4, Syn 10/8, Syn 12/4, and Syn 12/8, with the
following rules. For each data set, we set a (random) {\em seed}
string of the same length as strings in the data set from the same
alphabet. For instance, for the ``Syn 10/4'' data set with 10 bits
strings where each bit can take 4 possible values, 3232411132 is the
seed. (Here, the size of the alphabet is 4. Each 10-bit string denotes
a data instance with 10 features \stt{} each feature takes its value
from \{1,2,3,4\}.) A string $s$ in the data set is labelled positive
\ifaf{} $s$
match bits in the seed for exactly five places. E.g.,
\UL{3}1\UL{3}3\UL{4}2\UL{1}24\UL{2}\footnote{The underlined bits are
  identical to the seed.} is positive and
\UL{3}1\UL{3}3\UL{4}2\UL{1}2\UL{3}\UL{2} is negative (it shares 6 bits
as the seed rather than 5).
For each string classified as positive, we compute a $k$-bit
explanation. An explanation is {\em correct} \ifaf{} the seed string
has the same values for the bits identified as the explanation. The
accuracy of an explanation is defined as the number of correct bits
over the length of explanation. For instance, for $k=5$, we have\\
\centerline{
\begin{tabular}{cccc}
Query & Explanation & Seed & Accuracy \\
\hline
3233112143 & 323--1-1-- & 3232411132 & 1.0 \\
3244341112 & -2---411-2 & 3232411132 & 0.8 \\
\end{tabular}
}
The 2nd query contains an incorrect explanation 4. On
our synthetic data sets with a 70\% to 30\% split on training
and testing, the classification result is shown in Table~\ref{table:class2}
and the explanation accuracy for the \emph{Direct} and SHAP approaches is
shown in Table~\ref{table:exp}. This is an informative experiment as:
(1) there is no ``useless'' feature in the data set as every feature (bit)
could be decisive thus functions as part of an explanation as long as
its value is the same as the feature in the seed; (2) the seed is the
known ground truth for explanation comparison; (3) moreover, as shown
in Table~\ref{table:class2}, these datasets represent non-trivial
classification problems.

\begin{table}[!h]
\caption{Experiment results (F$_1$ scores) with synthetic data sets and
  several baseline algorithms. \label{table:class2}}
\begin{scriptsize}
\centerline{
\begin{tabular}{|r|cccc|}
\hline
       & Syn 10/4  & Syn 10/8 & Syn 12/4 & Syn 12/8 \\
\hline
Tree   & 0.71      & 0.78     & 0.62      & 0.70      \\
Direct & {\bf0.92} & 0.95     & {\bf0.89} & {\bf0.94} \\
\hline
CART   & 0.79      & 0.87     & 0.70      & 0.84       \\
MLP    & 0.77      & 0.83     & 0.73      & 0.80      \\
Forest & 0.90      & {\bf0.96}& 0.85      & 0.93      \\
SVM    & 0.85      & 0.86     & 0.81      & 0.81      \\
\hline
\end{tabular}
}
\end{scriptsize}
\end{table}


\begin{table}[!h]
\caption{Explanation accuracy on four syntactic data sets and
  various explanation lengths $k$.\label{table:exp}}
\begin{scriptsize}
\centerline{
\begin{tabular}{|c|c|c|c|c|c|c|}
    \hline
    \multicolumn{2}{|c|}{} & $k=1$ & $k=2$ & $k=3$ & $k=4$ & $k=5$\\
    \hline
    \multirow{2}{*}{10/4}&Direct&1&1&1&0.995&0.972\\
    \cline{2-7}
    &SHAP&1&1&0.996&0.993&0/962\\
    \hline
    \multirow{2}{*}{10/8}&Direct&1&1&0.997&0.980&0.976\\
    \cline{2-7}
    &SHAP&0.996&0.995&0.972&0.967&0.951\\
     \hline
    \multirow{2}{*}{12/4}&Direct&1&0.982&1&0.997&0.901\\
    \cline{2-7}
    &SHAP&0.993&0.980&0.973&0.942&0.856\\
     \hline
    \multirow{2}{*}{12/4}&Direct&1&1&0.998&0.975&0.964\\
    \cline{2-7}
    &SHAP&1&0.990&0.977&0.929&0.918\\
    \hline
\end{tabular}}
\end{scriptsize}
\label{tab:multicol}
\end{table}


Table~\ref{table:class2} shows that, similar to
Table~\ref{table:perfEvalReal}, the classification accuracy of
our approach is competitive comparing to the baseline approaches. This
further validates our approach for
classification. Table~\ref{table:exp} shows that although our
approach (Direct) and SHAP both can identify part of the seed string
from each query instances, hence computing correct explanations,
ours gives higher accuracy across the board.

To demonstrate the effect of knowledge incorporation, we gradually add
clauses drawn from sub-strings derived from the seed to the KB. The
result is shown in Figure~\ref{fig:KI}(a). Tested on the data set with
string length 10, size of alphabet 4 with the {\em Tree} algorithm,
we see that the classification performance improves as the number of
true clauses inserted grows.
To show that the knowledge incorporation is resilient to pollution, we
insert clauses of a random length between 1 and 10 with a random
probability to pollute the KB. As shown in Figure~\ref{fig:KI}(b),
for the same data set, the classification performance deteriorates
gradually as the number of random clauses grows.

\begin{figure}
\centerline{
\includegraphics[trim={0cm 2.1cm 0cm 0cm},
                       clip,width=9.8cm]{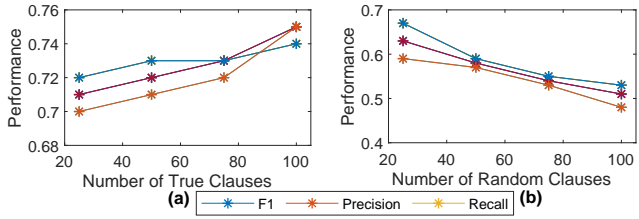}}
\caption{The plot on the left / right side shows
classification results from KBs with true / random clauses
inserted. \label{fig:KI}}
\vspace{-10pt}
\end{figure}


\section{Related Work}
\label{sec:relatedWork}

Probabilistic logic programming, or ProbLog, \cite{Fierens15} provides a
means to do logic programming with probabilities. Our work differs
from ProbLog in several ways. (1) ProbLog develops Logic Programming
and uses grounded predicates with closed world assumption to allow
negations whereas we use propositional clauses with classical negations;
(2) ProbLog uses Sato's distribution semantics and assumes all
atomic variables, the variables not derived with Logic Programming,
being independent whereas we use Nilsson's probabilistic logic
semantics and make no independence assumption. (3) ProbLog performs
inference with weighted model counting, which is then solved with
MAX-SAT, an NP-hard problem, whereas we use linear programming, which
is polynomial.

Performing probabilistic logic inference with mathematical programming
has been studied recently in \cite{Henderson19-aij} with its NonlInear
Probabilistic Logic Solver (NILS) approach. Although in both works
clauses with associated probabilities are turned into systems of
equations, the two approaches differ significantly.
NILS either assumes independence amongst its variables
or expand probability of
conjunctions as the product of the probability of a literal and some
conditionals.
Thus NILS produces
non-linear systems
and rely on gradient descent methods
for finding solutions. Consequently, NILS is unsuitable for
classification
as the independence assumption does not hold between the class labels
and feature values or, in general, values across different
features. When independence cannot be assumed, systems constructed
with NILS contains $k$th order equations with $2^k-1$ unknowns for each
$k$ literal clauses. Such high order equations with high number
unknowns are
difficult to solve numerically.
Comparing with NILS, the construction given in
Definition~\ref{dfn:lp} ``hides'' the complexity introduced by
conditionals in NILS with inequalities and ensures polynomial
complexity. Moreover, NILS does not tolerate inconsistency whereas
our approach does.

More generally, developing intelligent system based on reasoning with
KB has been explored in the past, see e.g.,
\cite{McCarthy68,Nilsson91}. Some of the early works on learning KB
from data use classical logic, e.g., \cite{Khardon94} or default logic
\cite{Roth96}. A comprehensive review on combining logic and
probability is beyond the scope of this section. For broader
discussions on this topic, see e.g.,
\cite{Bacchus90} for probabilistic logic, \cite{Chavira08} for weighted
model counting, and \cite{Gogate16} for probabilistic graphical
models with logical structures. The problem of testing a KB's
consistency is known as the probabilistic satisfiability (PSAT) problem. Works
dedicated to solving PSAT
include \cite{Cozman2015,Finger11,Georgakopoulos1988}. Since most of
these compute exact solutions over consistent KBs by solving an NP
problem, they are not suitable for classification.

In explainable machine learning, there has been
significant interest in providing explanations for classifiers; see
e.g., \cite{Biran17} for an overview. Works have been proposed to use
simpler thus weaker classifiers to explain results from stronger
ones, e.g., \cite{Feraud02}. Recent works on {\em
model-agnostic} explainers \cite{Ribeiro16,Ribeiro18} focus on adding
explanations to existing (black-box) classifiers. \cite{Alonso18} use
KB based classifiers to explain results obtained from MLP and random
forests. LIME \cite{Ribeiro16} augment the data with
randomly generated samples close to the instance to be explained and
then construct a simple thus explainable classifier to generate
explanations. \cite{Robnik08} works by decomposing a
model's predictions based on individual contributions of each feature.
\cite{Shih18} explains Bayesian network classifiers by compiling naive
Bayes and latent-tree classifiers into Ordered Decision Diagrams. \cite{shapley} provides explanations for decision trees based on the game-theoretic Shapley values.

\cite{Berrar2019}, \cite{Sachan18} and \cite{Vo17}
are some recent work on incorporating knowledge into machine
learning. \cite{Yu07} contains a survey,
categorising methods into four groups based on use of knowledge: (1)
to prepare training samples, (2) to initialise the hypothesis or
hypothesis space, (3) to alter the search objective and (4) to augment
the search process. Our approach fundamentally differs from those as
we represent knowledge in the same format as the model learned from
data and reason with both uniformly.

\section{Conclusion}
\label{sec:conclusion}

We present a non-parametric classification technique that gives
explanations to its predictions and supports knowledge
incorporation. Our approach is based on approximating literal
probabilities in probabilistic logic by solving linear systems
corresponding to KBs, which are either directly learned from data or
augmented with additional knowledge. Our linear program construction
is efficient and our approaches tolerate inconsistency in a KB.
As a stand-alone classifier, our approach matches or exceeds the
performance of existing algorithms on both synthetic and non-synthetic
data sets. At the same time, our approaches generate explanations in
the form of ``most decisive'' features. Upon comparing with a state of
the art Shapley Value based explanation method, SHAP, our approach
finds similar explanation as SHAP on real data sets. On four synthetic
data sets with known explanation ground truth, our approach is shown
to be superior as it achieves higher accuracy. Overall, we envisage
our approaches to be most useful for classification tasks where there
exists knowledge to complement data and explanations are required to
ensure usability.

There are four research directions that we plan to
explore. Firstly, this work focuses on developing the underlying
explainable classification techniques.
We will apply techniques developed
practical applications and perform user studies in the future.
Secondly, we will study
semantics for inconsistent KBs.
Thirdly,
we
will study richer explanation generation with with (probabilistic)
logic inference.
Lastly, we would like to develop other suitable
representations for knowledge incorporation.

\bibliographystyle{named}
\bibliography{bibliography}

\end{document}